\documentclass[times, review, 10pt]{elsarticle}
\usepackage{amsmath,amsfonts}
\usepackage{algorithmic}
\usepackage{array}
\usepackage[caption=false,font=normalsize,labelfont=sf,textfont=sf]{subfig}
\usepackage{textcomp}
\usepackage{url}
\usepackage{verbatim}
\usepackage{graphicx}

\RequirePackage{amsthm,amsmath,amsfonts,amssymb}
\RequirePackage[colorlinks,citecolor=blue,urlcolor=blue]{hyperref}
\RequirePackage{graphicx}
\usepackage{multirow}
\usepackage{hyperref}
\usepackage{tikz}
\usetikzlibrary{trees}
\theoremstyle{plain}

\newtheorem{theorem}{\textit{Theorem}}[]
\newtheorem{lemma}{\textit{Lemma}}[]
\theoremstyle{remark}
\newtheorem{definition}{\textbf{Definition}}[]

\usepackage{amsmath, amssymb, algorithm, algorithmic, graphicx}
\usepackage{tikz}
\usetikzlibrary{shapes.geometric, arrows}

\journal{Journal of \LaTeX\ Templates}

\begin{document}

\begin{frontmatter}

\title{Centroid Decision Forest}
\tnotetext[mytitlenote]{Fully documented templates are available in the elsarticle package on \href{http://www.ctan.org/tex-archive/macros/latex/contrib/elsarticle}{CTAN}.}

\author[add1]{Amjad Ali}
\author[add2,add3]{Hailiang Du}
\author[add1]{Saeed Aldahmani}
\author[add1]{Zardad Khan \corref{mycorrespondingauthor}}
\cortext[mycorrespondingauthor]{Corresponding author}
\ead{zaar@uaeu.ac.ae}

\address[add1]{Department of Statistics and Business Analytics, United Arab Emirates University, Al Ain, 15551, Abu Dhabi, UAE}
\address[add2]{Department of Mathematical Sciences, Durham University, Stockton Rd, Durham, UK}
\address[add3]{School of Mathematics, East China University of Science and Technology, Shanghai, 200237, China}

\begin{abstract}
This paper introduces the centroid decision forest (CDF), a novel ensemble learning framework that redefines the splitting strategy and tree building in the ordinary decision trees for high-dimensional classification. The splitting approach in CDF differs from the traditional decision trees in theat the class separability score (CSS) determines the selection of the most discriminative features at each node to construct centroids of the partitions (daughter nodes). The splitting criterion uses the Euclidean distance measurements from each class centroid to achieve a splitting mechanism that is more flexible and robust. Centroids are constructed by computing the mean feature values of the selected features for each class, ensuring a class-representative division of the feature space. This centroid-driven approach enables CDF to capture complex class structures while maintaining interpretability and scalability. To evaluate CDF, 23 high-dimensional datasets are used to assess its performance against different state-of-the-art classifiers through classification accuracy and Cohen’s kappa statistic. The experimental results show that CDF outperforms the conventional methods establishing its effectiveness and flexibility for high-dimensional classification problems.
\end{abstract}

\begin{keyword}
\texttt{High-Dimensional Data, Class Separability Score (CSS), Centroid based Partitioning, Decision Trees, Ensemble Learning.}
\end{keyword}

\end{frontmatter}

\section{Introduction}
Machine learning faces significant challenges when dealing with high-dimensional data types which occur frequently in fields including genomics, image analysis and financial modeling because feature counts tend to surpass available sample sizes \cite{bai2023joint}. The ``curse of dimensionality'' generates exponential volume expansion in the feature space combined with sparse data distributions that make reliable inference along with prediction tasks progressively harder \cite{wang2018efficient}. The performance of traditional machine learning techniques declines as data dimensionality increases because these methods require training data quantities to scale exponentially, as noted in \cite{ibrahim2021feature}. Multiple methods have been established by researchers to handle high-dimensional space issues such as feature selection along with regularization methods and novel distance metrics for reducing their detrimental effects \cite{mafarja2023classification}. Dimensionality reduction, combined with advances in quantum computing, presents new potential approaches to overcoming major challenges in this field, as discussed in \cite{huang2020predicting, peters2021machine}. Optimization approaches keep adapting by establishing equilibrium between model sophistication and generalization outcomes to achieve reliable and understandable solutions \cite{li2021novel}.

A central challenge in high-dimensional settings is the presence of redundant or irrelevant features, which not only increase computational complexity but can also obscure meaningful patterns in the data and impair model generalization \cite{johnstone2009statistical}. For instance, in gene expression studies, thousands of genes may be measured simultaneously, yet only a small subset is typically informative for predicting disease states \cite{fan2008sure}. Similarly, in image analysis, high-dimensional pixel data often contain significant noise and redundancy, complicating the extraction of discriminative features \cite{fisseha975high}. There has been a long discussion in the literature about these issues as evident in some of the important works like \cite{bickel2009simultaneous}, where the theoretical and practical implications of high dimensionality data analysis have been highlighted.

Furthermore, to address these issues, researchers have developed various techniques. Feature selection techniques are used to select and retain only the most informative features, using ranking criteria such as class separation or mutual information \cite{guyon2003introduction,li2017feature}. Various other methods, such as \cite{roffo2020infinite,komeili2020multiview,hou2023adaptive}, have also been proposed to address high-dimensional challenges effectively. In high dimensional settings, these methods are highly useful as they respectively reduce the size of the data while retaining its ability to discriminate amongst them. For example, in the area of genomics, feature selection techniques are applied for finding important biomarkers that are correlated with the underlying diseases, improving both prediction accuracy and interpretability of the predictive model \cite{shi2021feature, abdelwahab2022feature}. Principal component analysis (PCA) \cite{jolliffe2002principal} as well as $t$-distributed stochastic neighbour embedding ($t$-SNE) \cite{van2008visualizing} are dimensionality reduction methods that transform high dimensional data to a lower dimension so as to try to preserve its intrinsic structure. Still, these techniques might not be sufficient to fit complex nonlinearities between features and can omit useful discriminative information \cite{kingma2019introduction}. In recent years, deep learning-based approaches have emerged for automatic feature extraction, leveraging methods such as hierarchical convolutional factor analysis \cite{chen2013deep}, deep ReLU networks for feature extraction and generalization \cite{han2020depth}, and spatial sparseness modeling through deep networks \cite{chui2020realization}. Despite their impressive performance, these methods often require substantial computational cost and may suffer from a lack of interpretability \cite{bengio2013representation}.

Ensemble learning methods, such as random forests, have shown promise in improving classification performance by aggregating the predictions of multiple decision trees \cite{breiman2001random}. Random forests mitigate overfitting and improve generalization performance by introducing randomness through bootstrapping and selecting random subset of features for node splitting \cite{ho1998random}. The theoretical properties of random forests in high-dimensional settings have been rigorously investigated in several studies \cite{scornet2015consistency, biau2012analysis, wager2018estimation, chi2022asymptotic}. However, simple, threshold based splitting rules of traditional ensemble methods may not exploit the entire possible complexity of high-dimensional data \cite{cannings2017random}. The recent progress made in high-dimensional statistical inference and sparse recovery has helped to mitigate these issues \cite{buhlmann2002analyzing,lopes2019algorithmic}. Further studies focused on how elements of statistical learning theory and sparse recovery could facilitate effective learning in high-dimensional space with very few samples \cite{fan2015interaction, donoho2005sparse}. Also, machine learning techniques for high-dimensional tasks have greatly profited from the new developments in optimization and regularization. These methods make the models robust, simple, and efficient from the computation point of view \cite{biau2008consistency, Domingos2000BayesianAO, Mourtada2020optimal}. However, there is still an open problem of developing a systematic approach to deal with the intricate nature of high-dimensional spaces which would provide missing robustness and interpretability. Such a method should be capable of detecting highly informative features, understanding the structure of the data, and generalizing well to unseen data.

This work introduces the centroid decision forest (CDF), which aims to address the above issues by presenting a new ensemble learning approach that increases the classification accuracy in high-dimensional spaces. The CDF integrates multiple centroid decision trees (CDTs), where each tree is based on centroid based splitting using the most discriminative features. Its distinguishing nature is the selection and splitting of the feature space which focus on maximizing class classifiability and predictive accuracy. The CDT selects features that possess the highest class separability power from the give feature space at each node. This approach ensures that splits are performed on features with the highest discriminative power, enabling the model to effectively identify the underlying structure of the data. The CDF then calculates the centroids for the classes by averaging the corresponding feature values and filters the data with respect to Euclidean distances to the centroids. This enables dividing the data in a manner that aligns with the intrinsic structure of the dataset, resulting in more meaningful and interpretable data splits. In addition to the randomization induced by bootstrapping, a random subset of features is chosen adding further randomness that contributes to reducing overfitting and improving the model's capability for generalizing to unseen/new situation. Exploiting majority voting from multiple models built in this fashion leads to increased strength, robustness and low variance as compared to that of single model.

The proposed CDF ensemble is evaluated on 23 high-dimensional benchmark datasets, comparing its classification accuracy and Cohen’s kappa against state-of-the-art methods, i.e., classification and regression tree (CART) \cite{breiman1984classification}, random forest (RF) \cite{breiman2001random}, regularized random forest (RRF) \cite{deng2012feature,deng2013gene}, extreme gradient boosting (XGB) \cite{chen2016xgboost}, $k$ nearest neighbors ($k$NN) \cite{fukunaga1975k, keller1985fuzzy}, random $k$NN (R$k$NN) \cite{li2014random} and support vector machine (SVM) \cite{cortes1995support}. The results present that the proposed CDF ensemble consistently outperforms the classical procedures, showing the effectiveness for tackling complex classification tasks in high-dimensional spaces. The CDF provides a strong and efficient solution for high-dimensional data classification by combining ensemble learning, centroid based splitting, and creative feature selection.

The current paper is organized as follows: Section \ref{sec:methodology} gives a detailed description of the CDF method, mathematical background, and algorithmic implementation. Section \ref{sec:exp} presents the experimental design and datasets used for comparison. Section \ref{sec:results} presents the findings as well as comparison of CDF against the other techniques. Section \ref{sec:conclusion} concludes the article with an overview, limitations and possible future research directions.

\section{Methodology}
\label{sec:methodology}
The centroid decision forest (CDF) is an ensemble learning method that is specifically designed for high-dimensional classification problems. It integrates bootstrapping, feature subsets, class separability score (CSS) based feature selection and centroid based splitting to construct centroid decision trees (CDTs), resulting in a robust and accurate ensemble model. Unlike traditional decision trees that rely on threshold-based splits, the CDT partitions data based on class centroids, improving classification performance and stability.  

To formulate the problem, suppose a dataset $D = (X, Y)$ consisting of $n$ sample points, where $X \in \mathbb{R}^{n \times p}$ represents the feature matrix with $p$ features and $Y$ is the target variable with $K$ classes, i.e., $Y \in \{1, 2, \dots, K\}$. The objective is to learn a function:
$$f: \mathbb{R}^{n \times p} \rightarrow \{1, 2, \dots, K\},$$
that accurately predicts the class label for a new sample $X^*$. To achieve this, the following definitions are introduced in advance.

\begin{definition}\label{def1}
	\textbf{Class Separability Score (CSS)}\\
	The CSS quantifies the ability of a feature to distinguish between classes by comparing the means and variances of feature values across different classes in a pairwise manner. For a feature $j$, the CSS is computed as:
	
	\begin{equation}
		\text{CSS}_j = \frac{1}{\binom{K}{2}} \sum_{c' \ne c''} \frac{|\hat{\mu}_{c',j} - \hat{\mu}_{c'',j}|}{\hat{\sigma}_{c',j} + \hat{\sigma}_{c'',j} + \epsilon},
		\label{eq:css}
	\end{equation}
	
	where:
	\begin{itemize}
		\item $\hat{\mu}_{c,j} = \frac{1}{n_c} \sum_{i=1}^{n_c} X_{i,j}$ is the mean of feature $j$ for class $c$ with pairs ($c', c''$) defined below,
		\item $\hat{\sigma}_{c,j} = \sqrt{\frac{1}{n_c} \sum_{i=1}^{n_c} (X_{i,j} - \hat{\mu}_{c,j})^2}$ is the standard deviation of feature $j$ for class $c$ according the paired values,
		\item $\epsilon$ is a small constant (e.g., $10^{-7}$) to avoid division by zero,
		\item ($c', c''$) is a pair of 2 classes selected from $\{1, 2, \dots, K\}$. For $K$ classes, total number of pairs is $\binom{K}{2} = \frac{K(K-1)}{2}$.
	\end{itemize}
	
	The CSS measures the separability of classes by emphasizing features with large differences in means and small variances within classes. A higher CSS value indicates that the feature is more discriminative for classification.
\end{definition}

\begin{lemma}\label{lemma:css}
	\textbf{CSS and Feature Discriminability}\\
	If $\text{CSS}_j > \text{CSS}_k$ for two features $j$ and $k$, then feature $j$ is more discriminative than feature $k$ for classification.
\end{lemma}

\begin{proof}
	The CSS of a feature is directly proportional to the difference in means between classes and inversely proportional to the pooled standard deviation. A higher CSS indicates greater separability between classes, making the feature more discriminative.
\end{proof}

\begin{theorem}\label{thm:css_convergence}
	\textbf{Convergence of CSS based Feature Selection}\\
	As the number of samples $n_c$ for each class $c$ increases, the CSS of a feature $j$ converges to its expected value:
	$$
	\text{CSS}_j \to \mathbb{E}[\text{CSS}_j] \quad \text{as} \quad n_c \to \infty.
	$$
\end{theorem}

\begin{proof}
	By the \textit{Law of Large Numbers}, as $n_c \to \infty$, the sample mean $\hat{\mu}_{c,j}$ and the sample standard deviation $\hat{\sigma}_{c,j}$ converge to their population parameters, i.e.,  
	$$\hat{\mu}_{c,j} \to \mu_{c,j} \quad \text{and} \quad \hat{\sigma}_{c,j} \to \sigma_{c,j}.$$
	
	According to \textit{Slutsky’s theorem}, if a sequence of random variables $L_n$ and $M_n$ converges to constants $L$ and $M$, then their ratio also converges, i.e.,  
	$$\frac{L_n}{M_n} \to \frac{L}{M}, \quad (M \ne 0).$$
	
	Since the CSS of the $j$th feature is also a ratio of convergent statistics, i.e., differences in class means and pooled standard deviations each term in its computation converges to its expected value as $n_c \to \infty$.
	
\end{proof}

The selection process computes the CSS for all available features and selects the top $m$ features with the highest scores. Pseudocode for the CSS based feature selection is given in Algorithm \ref{css}.

\begin{algorithm}[H]
	\caption{CSS based feature selection.}
	\begin{algorithmic}[1]
		\REQUIRE Feature space $X$, class labels $Y$
		\ENSURE CSS values for all features to select top $m$ features
		\FOR{each feature $j \in \{1, 2, \dots, p\}$}
		\FOR{each class $c \in \{1, 2, \dots, K\}$}
		\STATE Compute mean $\hat{\mu}_{c,j}$ and standard deviation $\hat{\sigma}_{c,j}$
		\ENDFOR
		\STATE Initialize $\text{CSS}_j = 0$
		\FOR{each pair of classes $(c', c'')$ where $c' \ne c''$}
		\STATE Compute $|\hat{\mu}_{c',j} - \hat{\mu}_{c'',j}|$
		\STATE Compute $\hat{\sigma}_{c',j} + \hat{\sigma}_{c'',j} + \epsilon$
		\STATE Update $\text{CSS}_j = \text{CSS}_j + \frac{|\hat{\mu}_{c',j} - \hat{\mu}_{c'',j}|}{\hat{\sigma}_{c',j} + \hat{\sigma}_{c'',j} + \epsilon}$
		\ENDFOR
		\STATE Normalize $\text{CSS}_j = \frac{2}{K(K-1)} \text{CSS}_j$
		\ENDFOR
		\RETURN CSS values for all features
		\STATE Write CSS in descending order for all features
		\STATE Select the top $m$ features
	\end{algorithmic}
	\label{css}
\end{algorithm}

\begin{definition}\label{def2}
	\textbf{Centroid Based Splitting}\\
	Once the most discriminative features are selected, the data is split into partitions using centroid based splitting. For a class $c$, the centroid $C_c$ is defined as the mean of the feature vectors for all samples belonging to that class:
	
	\begin{equation}
		C_c = \frac{1}{n_c} \sum_{i=1}^{n_c} X_i,
		\label{eq:centroid}
	\end{equation}
	
	where $n_c$ is the number of samples in class $c$, and $X_i$ is the feature vector of the $i$th sample.
	
	The data is subsequently divided into partitions by allocating each data point to the closest centroid using to the Euclidean distance formula:
	
	\begin{equation}
		d(X_i, C_c) = \sqrt{\sum_{j=1}^{m} (X_{i,j} - C_{c,j})^2},
		\label{eq:distance}
	\end{equation}
	
	where $m$ stands for the number of dimensions.
	
	Centroid based partitioning, as a geometric method, provides a clear geometric interpretation of decision boundaries. When data points are assigned to the nearest centroid, hyperplanes implicitly separate the classes. In terms of computational efficiency, this method is optimal and well suited for the CSS based feature selection, as the centroids are determined using the most discriminative features.
\end{definition}

\begin{theorem}\label{thm:optimality}
	\textbf{Optimality of Centroid Based Partitioning}\\
	Consider the feature space $X \in \mathbb{R}^{n \times p}$ is associated with $K$ classes. Each class $c \in \{1, 2, \dots, K\}$ has a centroid $C_c$, computed as the mean of the feature vectors belonging to that class. Let $n_c$ be the number of data points in class $c$. Each observation $X_i$ is assigned to the class whose centroid is closest in Euclidean distance. The within-class sum of squares (WCSS) is defined as:
	$$\text{WCSS} = \sum_{c=1}^K \sum_{i=1}^{n_c} \|X_i - C_c\|^2.$$
	The optimal class assignment, which minimizes WCSS, is achieved when each observation is assigned to its closest centroid.
\end{theorem}

\begin{proof}
	Assigning each data point $X_i$ to nearest class centroid that minimizes WCSS. The $c$th class centroid $C_c$ is defined as the mean of all points in that class as defined in Equation \ref{eq:centroid}.
	Expanding the squared Euclidean norm:
	$$
	\|X_i - C_c\|^2 = X_i^T X_i - 2 X_i^T C_c + C_c^T C_c.
	$$
	Substituting this into the WCSS definition:
	$$
	\text{WCSS} = \sum_{c=1}^K \sum_{i=1}^{n_c} \left( X_i^T X_i - 2 X_i^T C_c + C_c^T C_c \right).
	$$
	Distributing the summation:
	\begin{equation}
		\text{WCSS} = \sum_{c=1}^K \left( \sum_{i=1}^{n_c} X_i^T X_i - 2 \sum_{i=1}^{n_c} X_i^T C_c + \sum_{i=1}^{n_c} C_c^T C_c \right).
		\label{wc}
	\end{equation}
	Since $C_c$ is the mean of the points in the $c$th class defined in Equation \ref{eq:centroid}, therefore,
	$$
	\sum_{i=1}^{n_c} X_i = n_c C_c.
	$$
	Rewriting the second terms in right hand side of Equation \ref{wc}:
	$$\sum_{i=1}^{n_c} -2 X_i^T C_c = -2 C_c^T \sum_{i=1}^{n_c} X_i,$$
	$$\implies \sum_{i=1}^{n_c} -2 X_i^T C_c = -2 C_c^T n_c C_c,$$ 
	\begin{equation}
		\implies \sum_{i=1}^{n_c} -2 X_i^T C_c = -2 n_c C_c^T C_c.\tag{i}
		\label{i}
	\end{equation}
	Similarly, for the last term,
	\begin{equation}
		\sum_{i=1}^{n_c} C_c^T C_c = n_c C_c^T C_c. \tag{ii}
		\label{ii}
	\end{equation}
	Use (\ref{i}) and (\ref{ii}) in Equation \ref{wc}. Thus, WCSS simplifies to,
	$$
	\text{WCSS} = \sum_{c=1}^K \left( \sum_{i=1}^{n_c} X_i^T X_i - 2 n_c C_c^T C_c + n_c C_c^T C_c \right).
	$$
	$$
	\implies \text{WCSS} = \sum_{c=1}^K \sum_{i=1}^{n_c} X_i^T X_i - \sum_{c=1}^K n_c C_c^T C_c.
	$$
	Minimizing WCSS is equivalent to maximizing
	$$
	\sum_{c=1}^K n_c C_c^T C_c.
	$$
	By the least squares property of the mean, the centroid $C_c$ minimizes the sum of squared distances within its assigned points. Any reassignment of points to a different class centroid increases the sum of squared distances.
	
	Thus, an observation's assignment to its closest class centroid is the optimal approach that minimizes WCSS.
\end{proof}

\begin{lemma}\label{lemma:stability}
	\textbf{Stability of Centroid Based Splitting}\\
	The centroid based splitting remains stable under small perturbations in the data. More specifically, if the data undergoes a small perturbation of $ \delta $, the resulting shift in the centroid is bounded by $ O(\delta) $.
\end{lemma}

\begin{proof}	
	Suppose the feature space $X$ in data $D=(X, Y)$ undergoes a small perturbation of at most $\delta$. Denoting the perturbed feature space as $X'$, the centroid given in Equation \ref{eq:centroid} for original data becomes, 
	$$
	C_c' = \frac{1}{n_c} \sum_{i=1}^{n_c} X_i'.
	$$
	Rewriting $ X_i' $ in terms of the original data and perturbation:
	$$
	X_i' = X_i + \epsilon_i, \quad \text{where } \|\epsilon_i\| \leq \delta.
	$$
	Thus, the perturbed centroid becomes:
	$$
	C_c' = \frac{1}{n_c} \sum_{i=1}^{n_c} (X_i + \epsilon_i).
	$$
	$$
	\implies C_c' = C_c + \frac{1}{n_c} \sum_{i=1}^{n_c} \epsilon_i.
	$$
	Taking norms:
	$$
	\|C_c' - C_c\| = \left\| \frac{1}{n_c} \sum_{i=1}^{n_c} \epsilon_i \right\|.
	$$
	Using the triangle inequality:
	$$
	\|C_c' - C_c\| \leq \frac{1}{n_c} \sum_{i=1}^{n_c} \|\epsilon_i\| \leq \frac{1}{n_c} \cdot n_c \delta = \delta.
	$$
	Thus, the change in the centroid is at most $ O(\delta) $, proving that centroid based splitting remains stable under small perturbations in the data.
\end{proof}

The centroid based partitioning algorithm computes centroids for each class and assigns data points to the nearest centroid. Pseudocode of the centroid based partitioning is given in Algorithm \ref{cbc}.

\begin{algorithm}[H]
	\caption{Centroid based partitioning.}
	\begin{algorithmic}[1]
		\REQUIRE Feature matrix $X$, class labels $Y$
		\ENSURE Partition assignments for all data points
		\FOR{each class $c \in \{1, 2, \dots, K\}$}
		\STATE Compute centroid $\mathbf{C}_c$ using Equation \ref{eq:centroid}
		\ENDFOR
		\FOR{each data point $\mathbf{X}_i$}
		\FOR{each class $c \in \{1, 2, \dots, K\}$}
		\STATE Compute distance $d(\mathbf{X}_i, \mathbf{C}_c)$ using Equation \ref{eq:distance}
		\ENDFOR
		\STATE Assign $\mathbf{X}_i$ to the partition with the smallest distance
		\ENDFOR
		\RETURN Partition assignments for all data points
	\end{algorithmic}
	\label{cbc}
\end{algorithm}

\subsection{Centroid Decision Tree (CDT)}\label{sec:cdt}
The centroid decision tree (CDT) is constructed recursively, starting from the root node and proceeding to split the data into partitions based on the most discriminative features and class centroids. Below, we describe the tree construction process in detail, including its mathematical foundations, recursive partitioning steps, and theoretical properties.

\begin{enumerate}
	\item \textbf{Start at the Root Node} \\
	The tree construction begins at the root node, which contains the entire training dataset $D = (X, Y)$.
	
	\item \textbf{Check Stopping Criteria} \\
	Before splitting a node, the algorithm checks if any of the following stopping criteria are met:
	\begin{itemize}
		\item \textbf{Maximum Depth Reached}: If the current depth of the tree equals the predefined maximum depth $d_{\text{max}}$, the node is declared a leaf node.
		\item \textbf{Minimum Samples Reached}: If the number of samples in the node is below a predefined threshold $n_{\text{min}}$, the node is declared a leaf node.
		\item \textbf{Pure Node}: If all samples in the node belong to the same class, the node is declared a leaf node.
	\end{itemize}
	
	If any of these conditions are met, the node is assigned the majority class of the samples in the node:
	\begin{equation}
		\text{Class}(\text{Leaf}) = \underset{\scriptstyle c \in \{1, 2, \dots, K\}}{\operatorname{arg-max}} \sum_{i=1}^{n_{\text{min}}} \mathbb{I}(y_i = c),
		\label{eq:leaf_class}
	\end{equation}
	where $\mathbb{I}(\cdot)$ is the indicator function.
	
	\item \textbf{Feature Selection} \\
	At each node, the algorithm selects the most discriminative features for splitting:
	\begin{itemize}
		\item A random subset of $m_{\text{try}}$ features is selected from the total $p$ features.
		\item The CSS is computed for each feature in the random subset of $m_{\text{try}}$ features using Equation \ref{eq:css} (Definition \ref{def1}).
		\item The top $m$ features with the highest CSS values are retained for splitting:
		
		\begin{equation}
			F_{\text{top}} = \underset{\scriptstyle F \subseteq \{1, 2, \dots, m_{\text{try}} \}, |F| = m}{\operatorname{arg-max}} \sum_{j \in F} \text{CSS}_j.
			\label{mfsel}
		\end{equation}
	\end{itemize}
	
	\item \textbf{Centroid Calculation} \\
	Using the selected features $F_{\text{top}}$, the algorithm computes the centroids for each class according to Equation \ref{eq:centroid} (Definition \ref{def2}). The centroids represent the central tendency of each class in the feature space.
	
	\item \textbf{Partitioning and Splitting} \\
	The data is split into partitions by assigning each sample to the nearest centroid based on the Euclidean distance:
	\begin{equation}
		\text{Partition}(X_i) = \underset{\scriptstyle c \in \{1, 2, \dots, K\}}{\operatorname{arg-min}} d(X_i, C_c),
		\label{eq:partitioning}
	\end{equation}
	where $d(X_i, C_c)$ is defined in Equation \ref{eq:distance}. This step divides the feature space into $K$ partitions, each corresponding to a class centroid.
	
	\item \textbf{Node Creation and Recursion} \\
	For each partition, a decision node is created. The algorithm is then recursively applied to each subset of data in the partition, creating subtrees. This process continues until the stopping criteria are met.
\end{enumerate}

Pseudocode for the construction of CDT is given in Algorithm \ref{cdt}.

\begin{algorithm}[H]
	\caption{Centroid decision tree (CDT) construction.}
	\begin{algorithmic}[1]
		\REQUIRE Feature matrix $X$, class labels $Y$, maximum depth $d_{\text{max}}$, minimum samples $n_{\text{min}}$, random subset of features tried at each node $m_{\text{try}}$, selected features $m$
		\ENSURE Decision tree ($\tau$)
		\IF{depth $\geq d_{\text{max}}$ or number of samples $\leq n_{\text{min}}$ or all samples belong to the same class}
		\STATE Create leaf node with majority class using Equation \ref{eq:leaf_class}
		\RETURN Leaf node
		\ENDIF
		\STATE Select top $m$ features using Equation \ref{mfsel} from the random subset of $m_{\text{try}}$ features (Algorithm \ref{css})
		\STATE Compute centroids $C_c$ for each class $c \in \{1, 2, \dots, K\}$ based on selected $m$ features using Equation \ref{eq:centroid} (Definition \ref{def2})
		\STATE Split data into partitions using centroid based splitting given in Equation \ref{eq:partitioning} (Algorithm \ref{cbc})
		\FOR{each partition}
		\STATE Create decision node for the partition
		\STATE Recursively apply algorithm to the subset of data in the partition
		\ENDFOR
		\RETURN Decision tree ($\tau$)
	\end{algorithmic}
	\label{cdt}
\end{algorithm}

\subsection{Centroid Decision Forest (CDF)}
\label{subsec:cdf}

The centroid decision forest (CDF) is an ensemble learning method that aggregates multiple centroid decision trees (CDTs) to improve classification performance and robustness. By combining the predictions of multiple CDTs, CDF reduces overfitting and enhances generalization, particularly in high-dimensional environments. The ensemble is constructed using bootstrapping and feature subset selection at each node, and predictions are aggregated through majority voting. The CDF is constructed as follows:

\begin{enumerate}
	\item \textbf{Bootstrapping}: Given a dataset $D = (X, Y)$ with $n$ samples, create $B$ bootstrap samples $D_1, D_2, \dots, D_B$, each of size $n$, by randomly sampling with replacement from $D$.
	
	\item \textbf{CDT Construction}: For each bootstrap sample $D_b$, construct a CDT using the methodology described in Subsection \ref{sec:cdt}. Each CDT is trained independently on its respective bootstrap sample.
	
	\item \textbf{Aggregation of Predictions}: For a new sample $X^*$, each CDT in the forest predicts a class label. The final prediction of the CDF is determined by majority voting:
	\begin{equation}
		\text{Prediction}(X^*) = \underset{\scriptstyle c \in \{1, 2, \dots, K\}}{\operatorname{arg-max}} \sum_{b=1}^B \mathbb{I}(f_b(X^*) = c).
	\end{equation}
	where:
	\begin{itemize}
		\item $f_b(X^*)$ is the prediction of the $b$th CDT,
		\item $\mathbb{I}(\cdot)$ is the indicator function.
	\end{itemize}
\end{enumerate}

The effectiveness of CDF stems from the following properties:
\begin{itemize}
	\item \textbf{Diversity}: By using bootstrapping and feature subset selection, the individual CDTs in the forest are diverse, reducing the risk of overfitting.
	\item \textbf{Robustness}: Aggregating predictions through majority voting reduces the impact of outliers and noisy data.
	\item \textbf{Scalability}: The construction of CDTs is computationally efficient, making CDF suitable for high-dimensional datasets.
\end{itemize}

The construction of the CDF is summarized in Algorithm \ref{cdf}.

\begin{algorithm}[H]
	\caption{Centroid decision forest (CDF) construction.}
	\begin{algorithmic}[1]
		\REQUIRE Feature matrix $X$, class labels $Y$
		\ENSURE Centroid decision forest ($\mathcal{F}$)
		\FOR{each bootstrap sample $b \in \{1, 2, \dots, B\}$}
		\STATE Create bootstrap sample $D_b$ by sampling $n$ instances from $D$ with replacement
		\STATE Construct centroid decision tree $\tau_b$ using $D_b$ (Algorithm \ref{cdt})
		\ENDFOR
		\RETURN Ensemble of trees $\mathcal{F} = \{\tau_1, \tau_2, \dots, \tau_B\}$
	\end{algorithmic}
	\label{cdf}
\end{algorithm}

Moreover, the proposed method has been implemented in the R package CDF, which is freely available on \href{https://cran.r-project.org/package=CDF}{CRAN}.

\section{Experimental Setup}
\label{sec:exp}
The experimental construction of the proposed centroid decision forest (CDF) ensemble is configured with the following parameter values. The ensemble is composed of 500 centroid decision trees (CDTs), with a constraint of at most 3 levels of depth for each tree. A node will be split only when there are more than 3 samples in that node and at least 2 different class labels. At each node, the splitting features are randomly selected from a subset comprising 20\% of the total features (i.e., $0.2 \times p$), where $p$ denotes the total number of features in the dataset. For the construction of a centroid, $2 \times log (p)$ features are selected using the class separability score (CSS). These parameters are kept constant in order to enhance computational efficiency, while the parameters of the competing methods are fine-tuned to achieve the best possible performance.

The proposed CDF ensemble is evaluated against eight state-of-the-art methodologies, CART, RF, RRF, XGB, $k$NN, R$k$NN, and SVM. To ensure a fair comparison, the parameters of all competing procedures are optimized using a 10-fold cross validation approach. For tree based methods, i.e., CART, RF, RRF, and XGB, key parameters such as the number of trees, maximum depth and leaf node size are fine-tuned. For $k$NN and R$k$NN, the number of neighbours is optimized, while for SVM, tuning involved selecting the kernel type and kernel specific parameters. This extensive parameter tuning ensures that each model can operate at its optimum level, enabling the most honest comparison of the novel CDF ensemble against the advanced techniques.

All benchmark datasets used in this study are randomly divided into 70\% training and 30\% testing portions. The proposed CDF and conventional models are trained on the training portion of the datasets and evaluated on the testing portion. This process is repeated 500 times to make a rigorous evaluation on the generalization capabilities of the models. 

\subsection{Datasets}
Table \ref{datasets} gives a brief overview of 23 benchmark datasets used in this study including major characteristics like data ID's, names, the number of features ($p$), the number of samples ($n$), the number of classes $K$, class distribution and source of the datasets. The datasets have highly different complexities, with varying number of features ($p$) ranging from 46 to 12,625 and samples ($n$) from 23 to 250.

\begin{table}[]
	\caption{Brief description of the benchmard datasets.}
	\begin{tabular}{llccccl}
		\hline
		\textbf{ID} & \textbf{Dataset} & \textbf{$p$} & \textbf{$n$} & \textbf{$K$} & \textbf{Class Distribution} & \textbf{Source} \\ \hline
		$D_1$ & Wind & 46 & 45 & 2 & 23/22 & \href{https://www.openml.org/search?type=data&status=active&id=785}{OpenML} \\
		$D_2$ & RobotFLP1 & 90 & 88 & 4 & 17/16/21/34 & \href{https://www.openml.org/search?type=data&status=active&id=1516}{OpenML} \\
		$D_3$ & RobotFLP4 & 90 & 117 & 3 & 72/24/21 & \href{https://www.openml.org/search?type=data&status=active&id=1519}{OpenML} \\
		$D_4$ & LSVT & 310 & 126 & 2 & 42/84 & \href{https://www.openml.org/search?type=data&status=active&id=1484}{OpenML} \\
		$D_5$ & DARWIN & 450 & 174 & 2 & 89/85 & \href{https://www.openml.org/search?type=data&status=any&id=46606}{OpenML} \\
		$D_6$ & Datascape & 536 & 230 & 2 & 115/115 & \href{https://www.kaggle.com/datasets/krishd123/high-dimensional-datascape/data}{Kaggle} \\
		$D_7$ & Breastcancer & 1000 & 250 & 2 & 192/58 & \href{https://CRAN.R-project.org/package=doBy}{CRAN} \\
		$D_8$ & Colon & 2000 & 62 & 2 & 40/22 & \href{https://www.openml.org/search?type=data\&status=any\&id=45087}{OpenML} \\
		$D_9$ & SRBCT & 2308 & 83 & 4 & 29/11/18/25 & \href{https://file.biolab.si/biolab/supp/bi-cancer/projections/info/SRBCT.html}{Biolab} \\
		$D_{10}$ & Lymphoma & 4026 & 45 & 2 & 23/22 & \href{https://www.openml.org/search?type=data&status=active&id=1101}{OpenML} \\
		$D_{11}$ & GSE2685 & 4522 & 30 & 2 & 8/22 & \href{https://file.biolab.si/biolab/supp/bi-cancer/projections/info/gastricGSE2685_2razreda.html}{Biolab} \\
		$D_{12}$ & DBWB & 4703 & 64 & 2 & 35/29 & \href{https://www.openml.org/search?type=data&sort=runs&id=1562&status=active}{OpenML} \\
		$D_{13}$ & DLBCL & 5469 & 77 & 2 & 58/19 & \href{https://www.openml.org/search?type=data\&status=active\&id=45088}{OpenML} \\
		$D_{14}$ & GSE89 & 5724 & 40 & 3 & 10/19/11 & \href{https://file.biolab.si/biolab/supp/bi-cancer/projections/info/bladderGSE89.html}{Biolab} \\
		$D_{15}$ & Prostate & 6033 & 102 & 2 & 50/52 & \href{https://CRAN.R-project.org/package=doBy}{CRAN} \\
		$D_{16}$ & Tumors-C & 7129 & 60 & 2 & 39/21 & \href{https://www.openml.org/search?type=data&status=active&id=1107}{OpenML} \\
		$D_{17}$ & Leukemia-3 & 7129 & 72 & 3 & 25/38/9 & \href{https://www.openml.org/search?type=data&status=active&id=45091}{OpenML} \\
		$D_{18}$ & Leukemia-4 & 7129 & 72 & 4 & 38/21/4/9 & \href{https://www.openml.org/search?type=data&status=active&id=45092}{OpenML} \\
		$D_{19}$ & GSE412 & 8280 & 110 & 2 & 50/60 & \href{https://file.biolab.si/biolab/supp/bi-cancer/projections/info/ALLGSE412_pred_poTh.html}{Biolab} \\
		$D_{20}$ & GSE967 & 9945 & 23 & 2 & 11/12 & \href{https://file.biolab.si/biolab/supp/bi-cancer/projections/info/EWSGSE967.html}{Biolab} \\
		$D_{21}$ & Lung & 12600 & 203 & 5 & 139/17/6/21/20 & \href{https://www.openml.org/search?type=data&status=active&id=45093}{OpenML} \\
		$D_{22}$ & Glioblastoma & 12625 & 50 & 4 & 14/7/14/15 & \href{https://file.biolab.si/biolab/supp/bi-cancer/projections/info/glioblastoma.html}{Biolab} \\
		$D_{23}$ & GSE2191 & 12625 & 54 & 2 & 28/26 & \href{https://file.biolab.si/biolab/supp/bi-cancer/projections/info/AMLGSE2191.html}{Biolab} \\
		\hline
	\end{tabular}
	\label{datasets}
\end{table}

\section{Results}
\label{sec:results}
\subsection{Benchmarking}
The experimental findings reveal that the centroid decision forest (CDF) ensemble can be regarded as truly remarkable among the evaluated 23 datasets. The CDF has unprecedented accuracy and beats all its rival techniques such as CART, RF, RRF, XGB, $k$NN, R$k$NN, and SVM. This high performance speaks strongly about the robustness of this CDF ensemble as well as its adaptability to datasets with different characteristics such as high dimensional, small sample size, number of classes, and imbalanced class distribution.

The CDF achieves the best accuracy on 18 out of 23 datasets, demonstrating its consistent superiority across a wide range of data complexities. As shown in Table \ref{acc}, CDF consistently delivers superior performance on datasets, i.e., $D_2$, $D_3$, $D_5$, $D_8$, $D_9$, $D_{10}$, $D_{11}$, $D_{12}$, $D_{13}$, $D_{15}$, $D_{16}$, $D_{17}$, $D_{18}$, $D_{19}$, $D_{20}$, $D_{21}$, $D_{22}$ and $D_{23}$. On datasets where CDF does not achieve the highest accuracy, i.e., $D_1$, $D_4$, $D_6$, $D_7$ and $D_{14}$; it remains highly competitive, often ranking second or third. For example, on $D_1$, XGB achieves the highest accuracy of 0.889, but CDF still achieves a strong accuracy of 0.858. Similarly, on $D_4$ and $D_6$, RF and XGB perform slightly better, but CDF remains close behind with accuracies of 0.785 and 0.948, respectively. On $D_7$, CDF achieves an accuracy of 0.829, while other methods such as XGB and RF perform slightly better.

In terms of Cohen's kappa, as presented in Table \ref{kap}, the proposed CDF ensemble achieves promising results, confirming its superiority in handling imbalanced datasets. The CDF method achieves the best kappa values on 17 out of 23 datasets, i.e., $D_2$, $D_3$, $D_5$, $D_8$, $D_9$, $D_{10}$, $D_{12}$, $D_{13}$, $D_{15}$, $D_{16}$, $D_{17}$, $D_{18}$, $D_{19}$, $D_{20}$, $D_{21}$, $D_{22}$, and $D_{23}$. On datasets where CDF does not perform well in terms of kappa, it still remains highly competitive. In the case of $D_1$, XGB achieves the highest kappa of 0.772, while CDF still achieves a strong kappa value of 0.710. In the case of $D_4$, RF performs slightly better with a kappa value of 0.592; however, CDF remains close behind with kappa scores of 0.480. A similar conclusion can be drawn for the $D_6$, $D_7$, $D_{11}$, and $D_{14}$ datasets, where the proposed method does not perform well.

The proposed CDF ensemble outperforms traditional classifiers, CART, $k$NN, and SVM as well as advanced ensemble methods, RF, RRF, and XGB, in most cases. As shown in the last row of Table \ref{acc}, CDF achieves an average accuracy of 0.871, compared to 0.735 for CART, 0.836 for RF, 0.778 for RRF, 0.812 for XGB, 0.754 for $k$NN, 0.764 for R$k$NN, and 0.715 for SVM. A similar pattern appears in the last row of Table \ref{kap}, where CDF achieves the highest average kappa score of 0.734, compared to 0.496 for CART, 0.663 for RF, 0.560 for RRF, 0.624 for XGB, 0.508 for $k$NN, 0.522 for R$k$NN, and 0.418 for SVM. These results highlight the robustness and stability of the proposed CDF in delivering more accurate and reliable prediction performance compared to the state-of-the-art classifiers.

Furthermore, the boxplots presented in Figures \ref{accbox} and \ref{kapbox} illustrate the distribution of classification accuracy and Cohen's kappa, respectively, across the benchmark datasets. The proposed CDF has very high median values along with the lowest variance among competing classification methods, proving its stability and robustness. The small interquartile range is also evident in the boxplots indicating that CDF has the ability to achieve a consistent classification result across various datasets. These concepts validate the performance of the centroid based feature selection and decision making implemented by CDF, making it well-suited for high-dimensional, imbalanced, and complex classification cases. Also, employing fixed value parameters simplifies model configuration and therefore ameliorates computational overhead while offering state-of-the-art performance.

\begin{table}[]
	\caption{Average classification accuracy of the proposed CDF and state-of-the-art methods across multiple datasets, based on 500 repeated training-testing splits.}
	\begin{tabular}{lcccccccc}
		\hline
		\multicolumn{1}{c}{\multirow{2}{*}{Dataset}} & \multicolumn{8}{c}{Method} \\
		\cline{2-9}
		\multicolumn{1}{c}{} & CDF & CART & RF & RRF & XGB & $k$NN & R$k$NN & SVM \\
		\hline
		$D_1$ & 0.858 & 0.872 & 0.885 & 0.869 & \textbf{0.889} & 0.726 & 0.759 & 0.830 \\
		$D_2$ & \textbf{0.881} & 0.689 & 0.831 & 0.789 & 0.742 & 0.743 & 0.758 & 0.589 \\
		$D_3$ & \textbf{0.920} & 0.747 & 0.889 & 0.833 & 0.860 & 0.765 & 0.831 & 0.742 \\
		$D_4$ & 0.785 & 0.765 & \textbf{0.831} & 0.807 & 0.824 & 0.621 & 0.662 & 0.670 \\
		$D_5$ & \textbf{0.887} & 0.739 & 0.878 & 0.819 & 0.858 & 0.706 & 0.873 & 0.844 \\
		$D_6$ & 0.948 & 0.942 & 0.953 & 0.947 & \textbf{0.971} & 0.937 & 0.863 & 0.924 \\
		$D_7$ & 0.829 & 0.790 & 0.827 & 0.822 & \textbf{0.836} & 0.822 & 0.816 & 0.821 \\
		$D_8$ & \textbf{0.838} & 0.724 & 0.817 & 0.761 & 0.764 & 0.744 & 0.728 & 0.772 \\
		$D_9$ & \textbf{0.992} & 0.777 & 0.990 & 0.828 & 0.926 & 0.876 & 0.870 & 0.854 \\
		$D_{10}$ & \textbf{0.969} & 0.777 & 0.908 & 0.804 & 0.840 & 0.803 & 0.809 & 0.761 \\
		$D_{11}$ & \textbf{0.950} & 0.799 & 0.927 & 0.843 & 0.898 & 0.869 & 0.938 & 0.780 \\
		$D_{12}$ & \textbf{0.878} & 0.774 & 0.859 & 0.783 & 0.813 & 0.560 & 0.585 & 0.520 \\
		$D_{13}$ & \textbf{0.964} & 0.821 & 0.861 & 0.838 & 0.865 & 0.873 & 0.892 & 0.797 \\
		$D_{14}$ & 0.840 & 0.609 & \textbf{0.849} & 0.716 & 0.686 & 0.775 & 0.818 & 0.603 \\
		$D_{15}$ & \textbf{0.903} & 0.848 & 0.888 & 0.848 & 0.900 & 0.826 & 0.827 & 0.860 \\
		$D_{16}$ & \textbf{0.674} & 0.556 & 0.604 & 0.579 & 0.580 & 0.607 & 0.618 & 0.636 \\
		$D_{17}$ & \textbf{0.945} & 0.805 & 0.895 & 0.822 & 0.916 & 0.839 & 0.789 & 0.696 \\
		$D_{18}$ & \textbf{0.882} & 0.800 & 0.834 & 0.777 & 0.861 & 0.804 & 0.665 & 0.637 \\
		$D_{19}$ & \textbf{0.890} & 0.830 & 0.769 & 0.806 & 0.879 & 0.670 & 0.727 & 0.733 \\
		$D_{20}$ & \textbf{0.918} & 0.360 & 0.802 & 0.676 & 0.674 & 0.637 & 0.673 & 0.444 \\
		$D_{21}$ & \textbf{0.919} & 0.820 & 0.899 & 0.845 & 0.915 & 0.914 & 0.856 & 0.868 \\
		$D_{22}$ & \textbf{0.731} & 0.511 & 0.678 & 0.552 & 0.625 & 0.686 & 0.655 & 0.568 \\
		$D_{23}$ & \textbf{0.627} & 0.539 & 0.563 & 0.535 & 0.557 & 0.546 & 0.565 & 0.491 \\
		\hline
		Average & \textbf{0.871} & 0.735 & 0.836 & 0.778 & 0.812 & 0.754 & 0.764 & 0.715\\
		\hline
	\end{tabular}
	\label{acc}
\end{table}

\begin{table}
	\caption{Average Cohen’s kappa values for the proposed CDF and state-of-the-art methods across multiple datasets, based on 500 repeated training-testing splits.}
	\begin{tabular}{lcccccccc}
		\hline
		\multicolumn{1}{c}{\multirow{2}{*}{Dataset}} & \multicolumn{8}{c}{Method} \\
		\cline{2-9}
		\multicolumn{1}{c}{} & CDF & CART & RF & RRF & XGB & $k$NN & R$k$NN & SVM \\
		\hline
		$D_1$ & 0.710 & 0.738 & 0.765 & 0.733 & \textbf{0.772} & 0.456 & 0.523 & 0.659 \\
		$D_2$ & \textbf{0.834} & 0.566 & 0.765 & 0.706 & 0.642 & 0.646 & 0.671 & 0.429 \\
		$D_3$ & \textbf{0.848} & 0.521 & 0.778 & 0.677 & 0.728 & 0.590 & 0.668 & 0.438 \\
		$D_4$ & 0.480 & 0.450 & \textbf{0.592} & 0.543 & 0.580 & 0.136 & 0.185 & 0.000 \\
		$D_5$ & \textbf{0.773} & 0.477 & 0.755 & 0.636 & 0.715 & 0.418 & 0.747 & 0.686 \\
		$D_6$ & 0.895 & 0.882 & \textbf{0.906} & 0.893 & 0.942 & 0.873 & 0.726 & 0.847 \\
		$D_7$ & 0.495 & 0.393 & 0.484 & 0.465 & \textbf{0.515} & 0.479 & 0.425 & 0.463 \\
		$D_8$ & \textbf{0.641} & 0.392 & 0.592 & 0.466 & 0.473 & 0.408 & 0.364 & 0.489 \\
		$D_9$ & \textbf{0.989} & 0.686 & 0.985 & 0.758 & 0.895 & 0.828 & 0.820 & 0.795 \\
		$D_{10}$ & \textbf{0.935} & 0.546 & 0.818 & 0.604 & 0.676 & 0.611 & 0.632 & 0.564 \\
		$D_{11}$ & 0.837 & 0.519 & 0.801 & 0.576 & 0.723 & 0.588 & \textbf{0.870} & 0.304 \\
		$D_{12}$ & \textbf{0.746} & 0.538 & 0.710 & 0.560 & 0.615 & 0.035 & 0.111 & 0.077 \\
		$D_{13}$ & \textbf{0.900} & 0.513 & 0.567 & 0.538 & 0.600 & 0.698 & 0.687 & 0.275 \\
		$D_{14}$ & 0.747 & 0.406 & \textbf{0.763} & 0.546 & 0.497 &.635 & 0.717 & 0.416 \\
		$D_{15}$ & \textbf{0.802} & 0.694 & 0.775 & 0.694 & 0.797 & 0.650 & 0.654 & 0.718 \\
		$D_{16}$ & \textbf{0.238} & 0.067 & 0.016 & 0.049 & 0.040 & 0.127 & 0.032 & 0.004 \\
		$D_{17}$ & \textbf{0.903} & 0.666 & 0.810 & 0.689 & 0.850 & 0.703 & 0.598 & 0.417 \\
		$D_{18}$ & \textbf{0.796} & 0.672 & 0.706 & 0.628 & 0.764 & 0.647 & 0.349 & 0.283 \\
		$D_{19}$ & \textbf{0.776} & 0.654 & 0.534 & 0.604 & 0.753 & 0.336 & 0.450 & 0.463 \\
		$D_{20}$ & \textbf{0.829} & 0.000 & 0.633 & 0.362 & 0.335 & 0.329 & 0.416 & 0.099 \\
		$D_{21}$ & \textbf{0.823} & 0.638 & 0.774 & 0.666 & 0.819 & 0.815 & 0.656 & 0.692 \\
		$D_{22}$ & \textbf{0.626} & 0.314 & 0.555 & 0.385 & 0.484 & 0.569 & 0.528 & 0.414 \\
		$D_{23}$ & \textbf{0.263} & 0.080 & 0.169 & 0.091 & 0.130 & 0.104 & 0.172 & 0.085 \\
		\hline
		Average & \textbf{0.734} & 0.496 & 0.663 & 0.560 & 0.624 & 0.508 & 0.522 & 0.418\\
		\hline
	\end{tabular}
	\label{kap}
\end{table}

\begin{figure}[h!]
	\centering
	\includegraphics[width=1\textwidth]{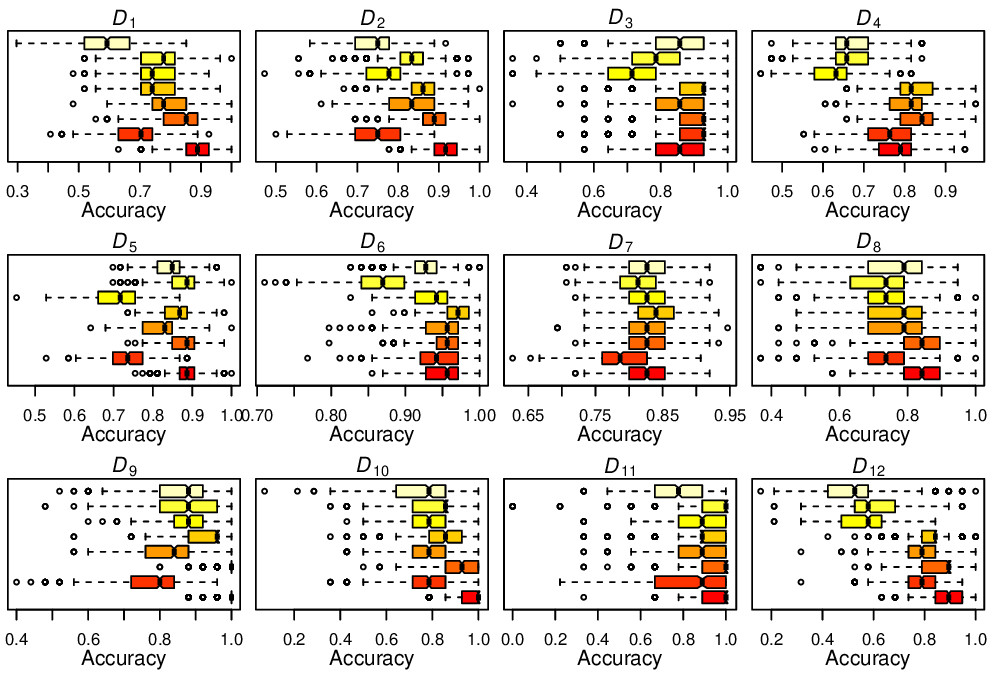}
	\includegraphics[width=1\textwidth]{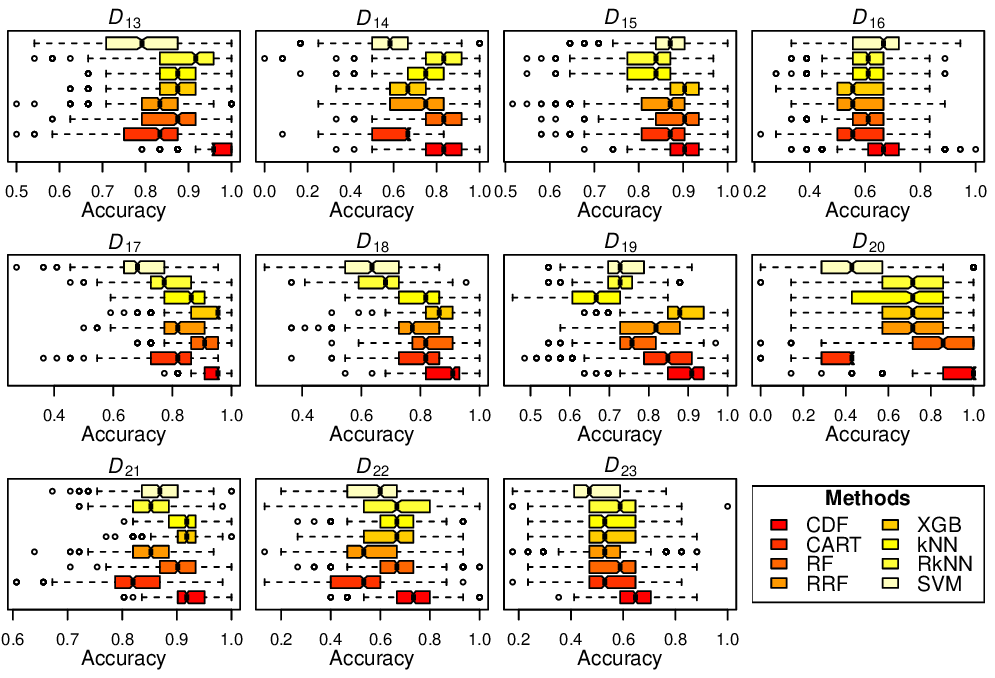}
	\caption{Boxplot of classification accuracy for the proposed CDF and state-of-the-art methods across multiple datasets. Accuracy is averaged over 500 repeated training-testing splits.}
	\label{accbox}
\end{figure}
\begin{figure}[h!]
	\centering
	\includegraphics[width=1\textwidth]{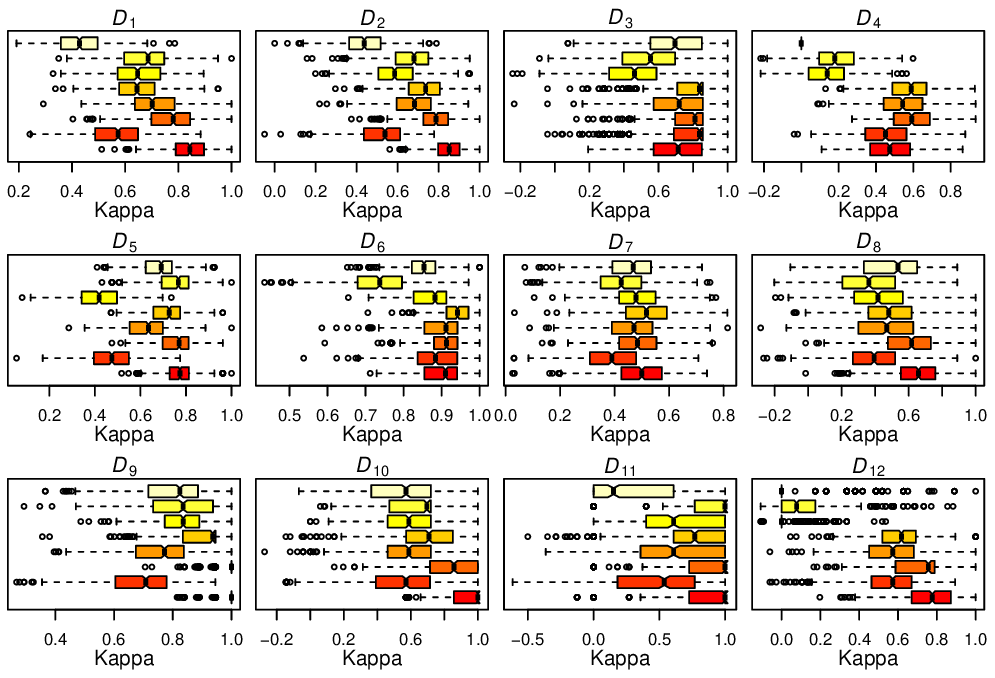}
	\includegraphics[width=1\textwidth]{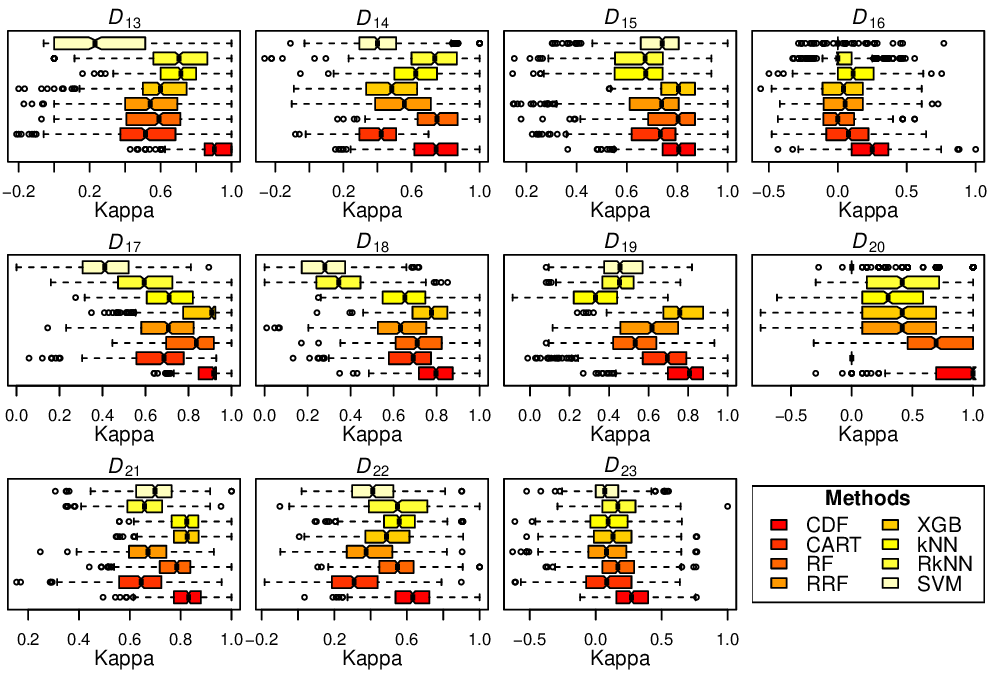}
	\caption{Boxplot of Cohen’s kappa for the proposed CDF and state-of-the-art methods across multiple datasets. Kappa values are averaged over 500 repeated training-testing splits}
	\label{kapbox}
\end{figure}

\subsection{Interpreting a CDT}
The centroid decision tree (CDT) in Figure \ref{centroid_tree} is constructed on $D_8$ (Colon) dataset. It selects the most discriminative features using the class separability score (CSS) and uses a centroid based splitting criterion to maximize class distinction at each node. Therefore, the combination of CSS for feature selection and centroid based partitioning gets rid of all other features that are not relevant for decision-making, thus improving classification accuracy.

At the root node, the features $X_{377}, X_{493}$ and $X_{1635}$ are selected based on their high CSS values, indicating their strong discriminatory power. To statistically validate their importance, the Wilcoxon rank-sum test is applied, yielding p-values of $1.99 \times 10^{-8}$ for $X_{377}$, $5.44 \times 10^{-8}$ for $X_{493}$ and $1.69 \times 10^{-7}$ for $X_{1635}$. These results confirm that the selected features significantly differentiate between the classes. Samples at the root node are split based on their proximity to class centroids. The left class centroid is $(0.96, 0.97, 0.91)$, while the right class centroid is $(-0.55, -0.57, -0.49)$. This centroid based decision mechanism ensures an optimal separation between the two classes, guiding samples toward their respective child nodes.

At the left child node, the model selects features $X_{1398}, X_{1570}$ and $X_{1924}$ to further refine class separation. The centroids for this node indicate distinct feature distributions, with left class centroid at $(0.56, -0.69, 0.71)$ and right class centroid at $(-0.91, 0.28, -0.91)$. These centroids are also constructed on the most significant features like in the root node, play a crucial role in distinguishing between the class distributions.

Similarly, at the right child node, the selected features $X_{1346}, X_{1466}$ and $X_{1772}$ further enhance the classification process. The left class centroid at this node is $(0.50, 0.46, 0.59)$, while the right class centroid is $(-2.24, -1.95, -0.94)$. These centroids highlight a strong feature separation, ensuring an effective decision boundary between the two classes.

At the final stage, the leaf nodes represent the ultimate classification decision using majority voting approach. Each sample follows a structured sequence of centroid based splits, progressively refining its class assignment. This approach guarantees high classification precision by leveraging the most informative features at each decision point.

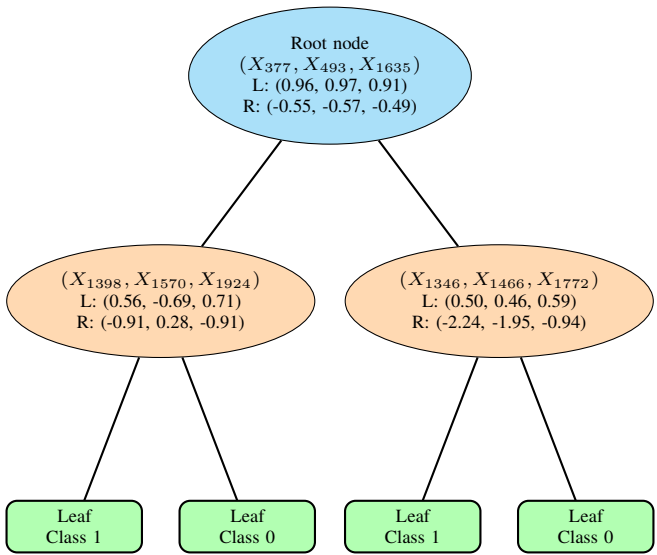
\begin{figure}[h!]
	\centering
	\begin{tikzpicture}[
		level distance=3cm, 
		level 1/.style={sibling distance=4.5cm}, 
		level 2/.style={sibling distance=2.3cm}, 
		edge from parent/.style={draw, thick},
		every node/.style={align=center, font=\scriptsize}, 
		root/.style={fill=cyan!30, draw=black, ellipse, minimum width=3cm, minimum height=1.5cm}, 
		internal/.style={fill=orange!30, draw=black, ellipse, minimum width=2.8cm, minimum height=1.5cm}, 
		leaf/.style={fill=green!30, draw=black, rounded corners, minimum width=1.8cm} 
		]
		
		\node[root] {Root node \\ 
			$(X_{377}, X_{493}, X_{1635})$ \\ 
			L: (0.96, 0.97, 0.91) \\ 
			R: (-0.55, -0.57, -0.49)}
		child { node[internal] {
				$(X_{1398}, X_{1570}, X_{1924})$ \\ 
				L: (0.56, -0.69, 0.71) \\ 
				R: (-0.91, 0.28, -0.91)}
			child { node[leaf] {Leaf \\ Class 1} }
			child { node[leaf] {Leaf \\ Class 0} }
		}
		child { node[internal] {
				$(X_{1346}, X_{1466}, X_{1772})$ \\ 
				L: (0.50, 0.46, 0.59) \\ 
				R: (-2.24, -1.95, -0.94)}
			child { node[leaf] {Leaf \\ Class 1} }
			child { node[leaf] {Leaf \\ Class 0} }
		};
	\end{tikzpicture}
	\caption{Structure of the CDT, illustrating selected features, centroids and splitting. At the first node, it selects top features (i.e., $X_{377}$, $X_{493}$) via CSS and splits data using class centroids (i.e., $(0.96, 0.97, 0.91)$ vs. $(-0.55, -0.57, -0.49)$). At each subsequent node, centroid-based partitioning refines separation, with Wilcoxon tests ($p < 0.001$) confirming feature significance. Final classification uses majority voting in the leaf nodes.}
	\label{centroid_tree}
\end{figure}

\subsection{Hyper-parameter Assessment}
The plot in Figure \ref{accplot} illustrates the relationship between the number of CDTs in the CDF and classification performance for $D_{13}$ (DLBCL) dataset. Initially, accuracy increases sharply as more trees are added, demonstrating the benefit of ensemble learning. Beyond approximately 300 trees, the accuracy stabilizes around 0.97, indicating diminishing returns, where adding more trees provides little to no improvement. Minor fluctuations are observed, but they remain within a narrow range. This pattern highlights the existence of an optimal tree count that balances accuracy and computational efficiency.

\begin{figure}[h!]
	\centering
	\includegraphics[width=1\textwidth]{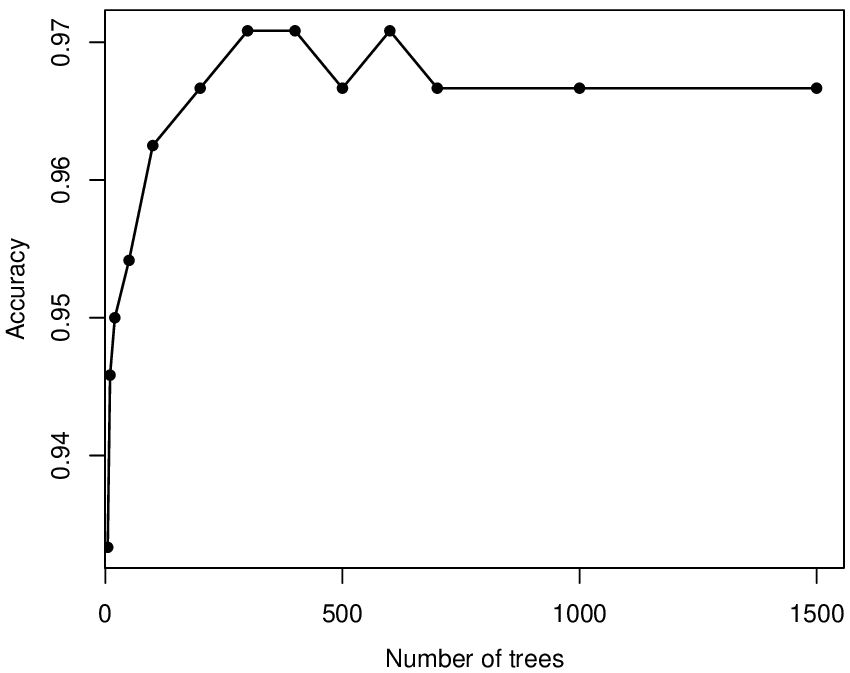}
	\caption{Classification accuracy in the CDF improves as more trees are added, stabilizing beyond 300 due to diminishing returns.}
	\label{accplot}
\end{figure}

The plot in Figure \ref{mtryplot} shows how classification accuracy changes with the percentage of randomly selected features used at each node. Initially, accuracy increases as more features are included, suggesting that a small but diverse subset enhances decision making. Beyond approximately 40\% of the total features per node, accuracy stabilizes near 0.98, indicating that additional features provide minimal improvement. This trend highlights the effectiveness of controlled feature selection at each node, ensuring optimal performance while maintaining computational efficiency.

\begin{figure}[h!]
	\centering
	\includegraphics[width=1\textwidth]{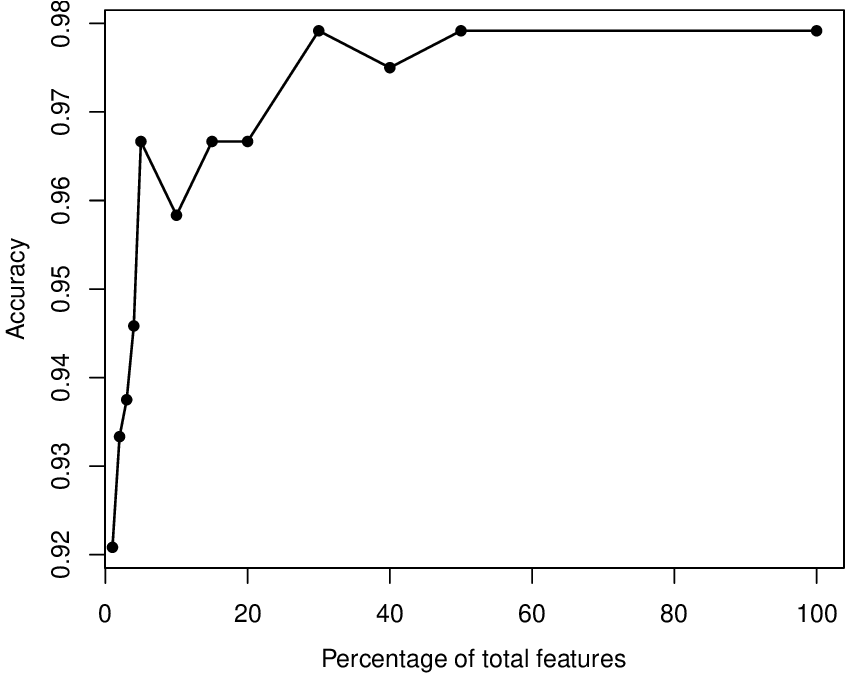}
	\caption{Impact of the percentage of randomly selected features per node on classification accuracy in the CDF.}
	\label{mtryplot}
\end{figure}

The plot in Figure \ref{cssplot} examines the relationship between the number of features selected using the class separability score (CSS) and classification accuracy for $D_{13}$ dataset. Accuracy initially rises as more features are included, reaching an early peak. However, beyond a certain point, accuracy stabilizes and eventually declines, suggesting that redundant or noisy features can negatively impact model performance. This highlights the importance of selecting an optimal feature subset to enhance node splitting and centroid based classification.

For the sake of simplicity, we have used fixed values of the above parameters. However, the above findings reveal that the CDF performance can further be improved by fine tuning of the hyper-parameters. 

\begin{figure}[h!]
	\centering
	\includegraphics[width=1\textwidth]{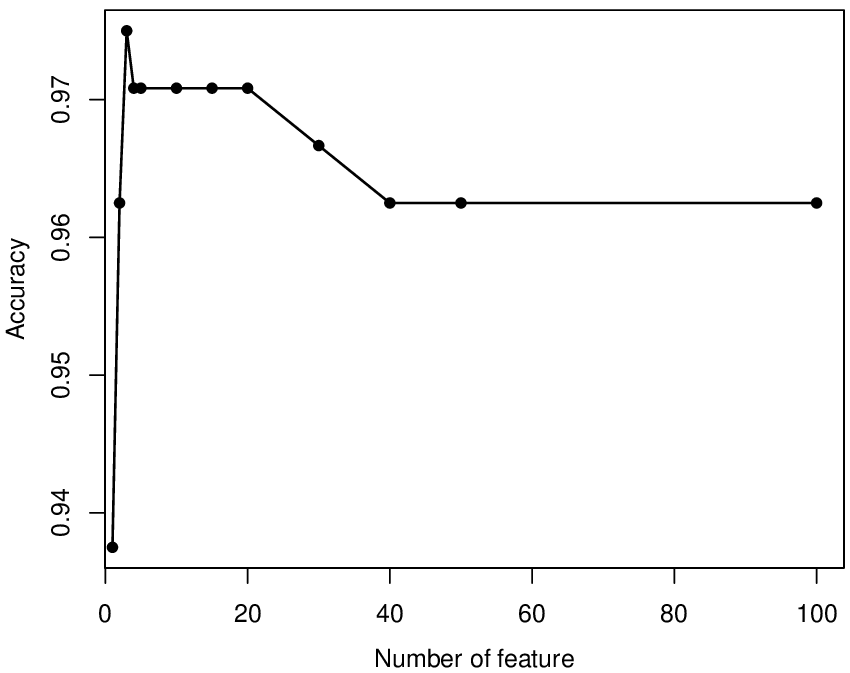}
	\caption{Effect of CSS selected features on accuracy in the CDF.}
	\label{cssplot}
\end{figure}

\section{Conclusion}
\label{sec:conclusion}
The centroid decision forest (CDF) introduced in this paper is an ensemble classifier designed for high-dimensional classification, particularly in small-sample settings. By redefining the splitting strategy in decision trees, CDF uses a class separability score (CSS) to select highly discriminative features and applies centroid-based partitioning of the feature space. This design allows the model to capture complex class structures while maintaining interpretability and scalability. The CSS helps to downweight irrelevant and noisy variables, and the combination of bootstrap sampling and feature sub-sampling yields a robust ensemble. Experimental results on multiple real high-dimensional datasets show that CDF consistently outperforms several state-of-the-art classifiers in terms of accuracy and Cohen’s kappa, highlighting its effectiveness for diverse and complex pattern recognition problems.

The proposed method is particularly useful for applications where the number of features greatly exceeds the number of observations, such as gene-expression and other omics data, high-resolution imaging, and certain financial problems. In these scenarios, conventional classifiers often struggle with noisy and redundant features, whereas CDF provides a practical and interpretable alternative. An implementation in the R environment further facilitates reproducible research and makes it straightforward for researchers and practitioners to apply CDF in their own pattern recognition tasks.

Despite these strengths, the classifier has some limitations. Its reliance on Euclidean distance makes it less suitable for categorical or mixed-type data without appropriate preprocessing or alternative distance measures. Moreover, the current CSS-based feature selection is only one possible choice and may be refined or adapted to different data characteristics. Future work will therefore focus on extending CDF to more flexible distance measures, exploring alternative or hybrid feature selection strategies, improving scalability for very large datasets through parallel and distributed implementations, and carrying out further empirical and theoretical studies to better understand its behaviour in a wide range of real-world pattern recognition applications.

\section*{Funding}
This research was funded by the United Arab Emirates University under grant number 12B086.

\bibliography{mybibfile}

\begin{thebibliography}{10}
\expandafter\ifx\csname url\endcsname\relax
  \def\url#1{\texttt{#1}}\fi
\expandafter\ifx\csname urlprefix\endcsname\relax\def\urlprefix{URL }\fi
\expandafter\ifx\csname href\endcsname\relax
  \def\href#1#2{#2} \def\path#1{#1}\fi

\bibitem{bai2023joint}
L.~Bai, H.~Li, W.~Gao, J.~Xie, H.~Wang, A joint multiobjective optimization of
  feature selection and classifier design for high-dimensional data
  classification, Information Sciences 626 (2023) 457--473.

\bibitem{wang2018efficient}
Q.~Wang, T.-T. Nguyen, J.~Z. Huang, T.~T. Nguyen, An efficient random forests
  algorithm for high dimensional data classification, Advances in Data Analysis
  and Classification 12 (2018) 953--972.

\bibitem{ibrahim2021feature}
S.~Ibrahim, S.~Nazir, S.~A. Velastin, Feature selection using correlation
  analysis and principal component analysis for accurate breast cancer
  diagnosis, Journal of imaging 7~(11) (2021) 225.

\bibitem{mafarja2023classification}
M.~Mafarja, T.~Thaher, M.~A. Al-Betar, J.~Too, M.~A. Awadallah, I.~Abu~Doush,
  H.~Turabieh, Classification framework for faulty-software using enhanced
  exploratory whale optimizer-based feature selection scheme and random forest
  ensemble learning, Applied Intelligence 53~(15) (2023) 18715--18757.

\bibitem{huang2020predicting}
H.-Y. Huang, R.~Kueng, J.~Preskill, Predicting many properties of a quantum
  system from very few measurements, Nature Physics 16~(10) (2020) 1050--1057.

\bibitem{peters2021machine}
E.~Peters, J.~Caldeira, A.~Ho, S.~Leichenauer, M.~Mohseni, H.~Neven,
  P.~Spentzouris, D.~Strain, G.~N. Perdue, Machine learning of high dimensional
  data on a noisy quantum processor, npj Quantum Information 7~(1) (2021) 161.

\bibitem{li2021novel}
Y.~Li, Y.~Chai, H.~Zhou, H.~Yin, A novel dimension reduction and dictionary
  learning framework for high-dimensional data classification, Pattern
  Recognition 112 (2021) 107793.

\bibitem{johnstone2009statistical}
I.~M. Johnstone, D.~M. Titterington, Statistical challenges of high-dimensional
  data, Philosophical transactions of the Royal Society A: Mathematical,
  physical and engineering sciences 367~(1906) (2009) 4237--4253.

\bibitem{fan2008sure}
J.~Fan, J.~Lv, Sure independence screening for ultrahigh dimensional feature
  space, Journal of the Royal Statistical Society Series B: Statistical
  Methodology 70~(5) (2008) 849--911.

\bibitem{fisseha975high}
G.~Fisseha~Gidey, C.~A. Onana, High dimensional data visualization: Advances
  and challenges, International Journal of Computer Applications 975  8887.

\bibitem{bickel2009simultaneous}
P.~J. Bickel, Y.~Ritov, A.~B. Tsybakov, Simultaneous analysis of lasso and
  dantzig selector, The Annals of Statistics 37~(4) (2009) 1705 -- 1732.
\newblock \href {https://doi.org/https://doi.org/10.1214/08-AOS620}
  {\path{doi:https://doi.org/10.1214/08-AOS620}}.

\bibitem{guyon2003introduction}
I.~Guyon, A.~Elisseeff, An introduction to variable and feature selection,
  Journal of machine learning research 3~(Mar) (2003) 1157--1182.

\bibitem{li2017feature}
J.~Li, K.~Cheng, S.~Wang, F.~Morstatter, R.~P. Trevino, J.~Tang, H.~Liu,
  Feature selection: A data perspective, ACM computing surveys (CSUR) 50~(6)
  (2017) 1--45.

\bibitem{roffo2020infinite}
G.~Roffo, S.~Melzi, U.~Castellani, A.~Vinciarelli, M.~Cristani, Infinite
  feature selection: a graph-based feature filtering approach, IEEE
  Transactions on Pattern Analysis and Machine Intelligence 43~(12) (2020)
  4396--4410.

\bibitem{komeili2020multiview}
M.~Komeili, N.~Armanfard, D.~Hatzinakos, Multiview feature selection for
  single-view classification, IEEE Transactions on Pattern Analysis and Machine
  Intelligence 43~(10) (2020) 3573--3586.

\bibitem{hou2023adaptive}
C.~Hou, R.~Fan, L.-L. Zeng, D.~Hu, Adaptive feature selection with augmented
  attributes, IEEE Transactions on Pattern Analysis and Machine Intelligence
  45~(8) (2023) 9306--9324.

\bibitem{shi2021feature}
Z.~Shi, B.~Wen, Q.~Gao, B.~Zhang, Feature selection methods for protein
  biomarker discovery from proteomics or multiomics data, Molecular \& Cellular
  Proteomics 20 (2021).

\bibitem{abdelwahab2022feature}
O.~Abdelwahab, N.~Awad, M.~Elserafy, E.~Badr, A feature selection-based
  framework to identify biomarkers for cancer diagnosis: A focus on lung
  adenocarcinoma, Plos one 17~(9) (2022) e0269126.

\bibitem{jolliffe2002principal}
I.~T. Jolliffe, Principal Component Analysis, 2nd Edition, Springer Series in
  Statistics, Springer New York, NY, 2002.
\newblock \href {https://doi.org/10.1007/b98835} {\path{doi:10.1007/b98835}}.

\bibitem{van2008visualizing}
L.~Van~der Maaten, G.~Hinton, Visualizing data using $t$-{SNE}., Journal of
  machine learning research 9~(11) (2008).

\bibitem{kingma2019introduction}
D.~P. Kingma, M.~Welling, et~al., An introduction to variational autoencoders,
  Foundations and Trends{\textregistered} in Machine Learning 12~(4) (2019)
  307--392.

\bibitem{chen2013deep}
B.~Chen, G.~Polatkan, G.~Sapiro, D.~Blei, D.~Dunson, L.~Carin, Deep learning
  with hierarchical convolutional factor analysis, IEEE transactions on pattern
  analysis and machine intelligence 35~(8) (2013) 1887--1901.

\bibitem{han2020depth}
Z.~Han, S.~Yu, S.-B. Lin, D.-X. Zhou, Depth selection for deep relu nets in
  feature extraction and generalization, IEEE Transactions on Pattern Analysis
  and Machine Intelligence 44~(4) (2020) 1853--1868.

\bibitem{chui2020realization}
C.~K. Chui, S.-B. Lin, B.~Zhang, D.-X. Zhou, Realization of spatial sparseness
  by deep relu nets with massive data, IEEE Transactions on Neural Networks and
  Learning Systems 33~(1) (2020) 229--243.

\bibitem{bengio2013representation}
Y.~Bengio, A.~Courville, P.~Vincent, Representation learning: A review and new
  perspectives, IEEE transactions on pattern analysis and machine intelligence
  35~(8) (2013) 1798--1828.

\bibitem{breiman2001random}
L.~Breiman, Random forests, Machine learning 45 (2001) 5--32.

\bibitem{ho1998random}
T.~K. Ho, The random subspace method for constructing decision forests, IEEE
  transactions on pattern analysis and machine intelligence 20~(8) (1998)
  832--844.

\bibitem{scornet2015consistency}
E.~Scornet, G.~Biau, J.-P. Vert, Consistency of random forests, The Annals of
  Statistics 43~(4) (2015) 1716 -- 1741.
\newblock \href {https://doi.org/https://doi.org/10.1214/15-AOS1321}
  {\path{doi:https://doi.org/10.1214/15-AOS1321}}.

\bibitem{biau2012analysis}
G.~Biau, Analysis of a random forests model, The Journal of Machine Learning
  Research 13 (2012) 1063--1095.

\bibitem{wager2018estimation}
S.~Wager, S.~Athey, Estimation and inference of heterogeneous treatment effects
  using random forests, Journal of the American Statistical Association
  113~(523) (2018) 1228--1242.

\bibitem{chi2022asymptotic}
C.-M. Chi, P.~Vossler, Y.~Fan, J.~Lv, Asymptotic properties of high-dimensional
  random forests, The Annals of Statistics 50~(6) (2022) 3415--3438.

\bibitem{cannings2017random}
T.~I. Cannings, R.~J. Samworth, Random-projection ensemble classification,
  Journal of the Royal Statistical Society Series B: Statistical Methodology
  79~(4) (2017) 959--1035.

\bibitem{buhlmann2002analyzing}
P.~B{\"u}hlmann, B.~Yu, Analyzing bagging, The annals of Statistics 30~(4)
  (2002) 927--961.

\bibitem{lopes2019algorithmic}
M.~E. Lopes, Estimating the algorithmic variance of randomized ensembles via
  the bootstrap, The Annals of Statistics 47~(2) (2019) 1088 -- 1112.
\newblock \href {https://doi.org/https://doi.org/10.1214/18-AOS1707}
  {\path{doi:https://doi.org/10.1214/18-AOS1707}}.

\bibitem{fan2015interaction}
Y.~Fan, Y.~Kong, D.~Li, Z.~Zheng, Innovated interaction screening for
  high-dimensional nonlinear classification, The Annals of Statistics 43~(3)
  (2015) 1243 -- 1272.
\newblock \href {https://doi.org/https://doi.org/10.1214/14-AOS1308}
  {\path{doi:https://doi.org/10.1214/14-AOS1308}}.

\bibitem{donoho2005sparse}
D.~L. Donoho, J.~Tanner, Sparse nonnegative solution of underdetermined linear
  equations by linear programming, Proceedings of the national academy of
  sciences 102~(27) (2005) 9446--9451.

\bibitem{biau2008consistency}
G.~Biau, L.~Devroye, G.~Lugosi,
  \href{http://jmlr.org/papers/v9/biau08a.html}{Consistency of random forests
  and other averaging classifiers}, Journal of Machine Learning Research 9~(66)
  (2008) 2015--2033.
\newline\urlprefix\url{http://jmlr.org/papers/v9/biau08a.html}

\bibitem{Domingos2000BayesianAO}
P.~M. Domingos,
  \href{https://api.semanticscholar.org/CorpusID:17792327}{Bayesian averaging
  of classifiers and the overfitting problem}, in: International Conference on
  Machine Learning, 2000.
\newline\urlprefix\url{https://api.semanticscholar.org/CorpusID:17792327}

\bibitem{Mourtada2020optimal}
J.~Mourtada, S.~Ga{\"i}ffas, E.~Scornet, Minimax optimal rates for mondrian
  trees and forests, The Annals of Statistics 48~(4) (2020) 2253 -- 2276.
\newblock \href {https://doi.org/https://doi.org/10.1214/19-AOS1886}
  {\path{doi:https://doi.org/10.1214/19-AOS1886}}.

\bibitem{breiman1984classification}
L.~Breiman, J.~Friedman, R.~A. Olshen, C.~J. Stone, Classification and
  Regression Trees, 1st Edition, Chapman and Hall/CRC, 1984, eBook published 19
  October 2017.
\newblock \href {https://doi.org/10.1201/9781315139470}
  {\path{doi:10.1201/9781315139470}}.

\bibitem{deng2012feature}
H.~Deng, G.~Runger, Feature selection via regularized trees, in: The 2012
  International Joint Conference on Neural Networks (IJCNN), IEEE, 2012, pp.
  1--8.

\bibitem{deng2013gene}
H.~Deng, G.~Runger, Gene selection with guided regularized random forest,
  Pattern recognition 46~(12) (2013) 3483--3489.

\bibitem{chen2016xgboost}
T.~Chen, C.~Guestrin, {XGBoost}: A scalable tree boosting system, in:
  Proceedings of the 22nd acm sigkdd international conference on knowledge
  discovery and data mining, 2016, pp. 785--794.

\bibitem{fukunaga1975k}
K.~Fukunaga, L.~Hostetler, K-nearest-neighbor {Bayes}-risk estimation, IEEE
  Transactions on Information Theory 21~(3) (1975) 285--293.

\bibitem{keller1985fuzzy}
J.~M. Keller, M.~R. Gray, J.~A. Givens, A fuzzy k-nearest neighbor algorithm,
  IEEE transactions on systems, man, and cybernetics~(4) (1985) 580--585.

\bibitem{li2014random}
S.~Li, E.~J. Harner, D.~A. Adjeroh, Random knn, in: 2014 IEEE International
  Conference on Data Mining Workshop, IEEE, 2014, pp. 629--636.

\bibitem{cortes1995support}
C.~Cortes, V.~Vapnik, Support-vector networks, Machine learning 20 (1995)
  273--297.

\end{thebibliography}
\end{document}